%% file: ms.tex
\pdfoutput=1

\documentclass[sigconf,authorversion,screen]{acmart}

\AtBeginDocument{%
  \providecommand\BibTeX{{%
    \normalfont B\kern-0.5em{\scshape i\kern-0.25em b}\kern-0.8em\TeX}}}

\copyrightyear{2021}
\acmYear{2021}
\setcopyright{acmlicensed}
\acmConference[SIGIR '21]{Proceedings of the 44th International ACM SIGIR Conference on Research and Development in Information Retrieval}{July 11--15, 2021}{Virtual Event, Canada}
\acmBooktitle{Proceedings of the 44th International ACM SIGIR Conference on Research and Development in Information Retrieval (SIGIR '21), July 11--15, 2021, Virtual Event, Canada}
\acmPrice{15.00}
\acmDOI{10.1145/3404835.3463084}
\acmISBN{978-1-4503-8037-9/21/07}

\input{commands} 

\begin{document}


\title{Propensity-scored Probabilistic Label Trees}


\author{Marek Wydmuch}
\affiliation{%
  \institution{Poznan University of Technology}
  \city{Poznan}
  \country{Poland}
}
\email{mwydmuch@cs.put.poznan.pl}

\author{Kalina Jasinska-Kobus}
\affiliation{%
  \institution{ML Research at Allegro.pl}
  \city{Poznan}
  \country{Poland}
}
\additionalaffiliation{
  \institution{Poznan University of Technology}
  \city{Poznan}
  \country{Poland}
}
\email{kjasinska@cs.put.poznan.pl}

\author{Rohit Babbar}
\affiliation{%
  \institution{Aalto University}
  \city{Helsinki}
  \country{Finland}
}
\email{rohit.babbar@aalto.fi}

\author{Krzysztof Dembczyński}
\authornotemark[1]
\affiliation{%
  \institution{Yahoo! Research}
  \city{New York}
  \country{USA}
}
\email{kdembczynski@cs.put.poznan.pl}



\begin{abstract}
Extreme multi-label classification (XMLC) refers to the task of tagging instances 
with small subsets of relevant labels coming from an extremely large set of all possible labels. 
Recently, XMLC has been widely applied to diverse web applications 
such as automatic content labeling, online advertising, or recommendation systems. 
In such environments, label distribution is often highly imbalanced, 
consisting mostly of very rare tail labels, 
and relevant labels can be missing.
As a remedy to these problems, 
the propensity model has been introduced and applied within several XMLC algorithms.
In this work, we focus on the problem of optimal predictions under this model
for probabilistic label trees, a popular approach for XMLC problems.
We introduce an inference procedure, based on the $A^*$-search algorithm, 
that efficiently finds the optimal solution,
assuming that all probabilities and propensities are known. 
We demonstrate the attractiveness of this approach in a wide empirical study on popular XMLC benchmark datasets. 

\end{abstract}

\maketitle


\section{Introduction}
\label{introduction}

Extreme multi-label classification (XMLC) is a supervised learning problem, where only a few labels from an enormous label space, reaching orders of millions, are relevant per data point. Notable examples of problems where XMLC framework can be effectively leveraged are tagging of text documents~\citep{Dekel_Shamir_2010}, content annotation for multimedia search~\citep{Deng_et_al_2011}, and diverse types of recommendation, including webpages-to-ads~\citep{Beygelzimer_et_al_2009b}, ads-to-bid-words~\citep{Agrawal_et_al_2013,Prabhu_Varma_2014}, users-to-items~\citep{Weston_et_al_2013, Zhuo_et_al_2020}, queries-to-items~\citep{Medini_et_al_2019}, or items-to-queries~\citep{Chang_et_al_2020}.
These practical applications impose new statistical challenges, including: 
1) long-tail distribution of labels---infrequent (tail) labels are much harder to predict than frequent (head) labels due to the data imbalance problem;  
2) missing relevant labels in learning data---since it is nearly impossible to check the whole set of labels when it is so large, and the chance for a label to be missing is higher for tail than for head labels~\citep{Jain_et_al_2016}.

Many XMLC models achieve good predictive performance by just focusing on head labels~\citep{Wei_Li_2018}. However, this is not desirable in many of the mentioned applications (e.g., recommendation and content annotation), 
where tail labels might be more informative.
To address this issue \citet{Jain_et_al_2016} proposed to evaluate XMLC models in terms of propensity-scored versions of popular measures (i.e., precision$@k$, recall$@k$, and nDCG$@k$). Under the propensity model, we assume that an assignment of a label to an example is always correct, but the supervision may skip some positive labels and leave them not assigned to the example with some probability (different for each label). 

In this work, we introduce the Bayes optimal inference procedure for propensity-scored precision$@k$ for probabilistic classifiers trained on observed data. While this approach can be easily applied to many classical models, we particularly show how to implement it for probabilistic label trees (\Algo{PLT}s)~\citep{Jasinska_et_al_2016}, an efficient and competitive approach to XMLC, being the core of many existing state-of-the-art algorithms
(e.g., \Algo{Parabel}~\citep{Prabhu_et_al_2018}, \Algo{extremeText}~\citep{Wydmuch_et_al_2018}, 
\Algo{Bonsai}~\citep{Khandagale_et_al_2019},
\Algo{AttentionXML}~\citep{You_et_al_2019}, \Algo{napkinXC}~\citep{Jasinska-Kobus_et_al_2020}, and \Algo{PECOS} that includes \Algo{XR-Linear}~\citep{Yu_et_al_2020} and \Algo{X-Transformers}~\citep{Chang_et_al_2020} methods).
We demonstrate that this approach achieves very competitive results in terms of statistical performance and running times.%


\section{Problem statement}
\label{sec:problem_statement}

In this section, we state the problem. We first define extreme multi-label classification (XMLC) and then the propensity model.

\subsection{Extreme multi-label classification}
\label{subsec:xmlc}

Let $\calX$ denote an instance space, and let $\labels = [m]$ be a finite set of $m$ class labels. 
We assume that an instance $\bx \in \calX$ is associated with a subset of labels $\labels_{\bx} \subseteq \calL$ 
(the subset can be empty); 
this subset is often called the set of \emph{relevant} or \emph{positive} labels,
while the complement $\labels \backslash \labels_{\bx}$ is considered as \emph{irrelevant} or \emph{negative} for $\bx$. 
We identify the set $\labels_{\bx}$ of relevant labels with the binary vector $\by = (y_1,y_2, \ldots, y_m)$, 
in which $y_j = 1 \Leftrightarrow j \in \labels_{\bx}$. 
By $\calY = \{0, 1\}^m$ we denote the set of all possible label vectors.
In the classical setting, we assume that observations $(\bx, \by)$ are generated independently and identically
according to a probability distribution $\prob(\bx, \by)$ defined on $\calX \times \calY$. 
Notice that the above definition concerns not only multi-label classification, but also multi-class (when $\|\by\|_1=1$) and $k$-sparse multi-label (when $\|\by\|_1\le k$) problems as special cases.
In case of XMLC we assume $m$ to be a large number 
(e.g., $\ge 10^5$), 
and $\|\by\|_1$ to be much smaller than $m$, $\|\by\|_1 \ll m$.%
\footnote{We use $[n]$ to denote the set of integers from $1$ to $n$, and $\|\bx\|_1$ to denote the $L_1$ norm of $\bx$.}

The problem of XMLC can be defined as finding a \emph{classifier} $\bh(\bx) = (h_1(\bx), h_2(\bx),\ldots, h_m(\bx))$, from a function class $\calH^m: \calX \rightarrow \mathbb{R}^m$, that minimizes the \emph{expected loss} or \emph{risk}:  
\begin{equation}
\loss_\ell(\bh) = \mathbb{E}_{(\bx,\by) \sim \prob(\bx,\by)} (\ell(\by, \bh(\bx))\,,
\end{equation}
where $\ell(\by, \hat{\by})$ is the  (\emph{task}) \emph{loss}.
The optimal classifier,  the so-called \emph{Bayes classifier},  for a given loss function $\ell$ is:
$
\bh^*_\ell = \argmin_{\bh}  \loss_\ell(\bh) \,.
$


\subsection{Propensity model}
\label{subsec:propensity_model}

In the case of XMLC, the real-world data may not follow the classical setting, which assumes that $(\bx, \by)$ are generated according to $\prob(\bx, \by)$.
As correct labeling (without any mistakes or noise) in case of an extremely large label set is almost impossible,
it is reasonable to assume that positive labels can be missing~\citep{Jain_et_al_2016}.
Mathematically, the model can be defined in the following way.
Let $\by$ be the original label vector associated with $\bx$. 
We observe, however, $\tby = (\ty_1, \ldots, \ty_m)$ such that:
\begin{equation}
\begin{array}{l l}
 \prob(\ty_j = 1 \given y_j = 1) = p_j\,, & \prob(\ty_j = 0 \given y_j = 1) = 1 - p_j \,,\\
 \prob(\ty_j = 1 \given y_j = 0) = 0  \,, &  \prob(\ty_j = 0 \given y_j = 0) = 1 \,, \\
\end{array}
\end{equation}
where $p_j \in [0,1]$ is the propensity of seeing a positive label when it is indeed positive.
All observations in both training and test sets do follow the above model.
The propensity does not depend on $\bx$. This means that for the observed 
conditional probability of label $j$,
%
we have:
\begin{equation}
\teta_j(\bx) = \prob(\ty_j = 1 \given \bx) =  p_j\prob(y_j = 1 \given \bx) = p_j \eta_j(\bx) \,.
\end{equation}
Let us denote the inverse propensity by $q_j$, i.e. $q_j = \frac{1}{p_j}$. 
Thus, the original 
conditional probability of label $j$ is given by:
\begin{equation}
\eta_j(\bx) = \prob(y_j = 1 \given \bx) =  q_j\prob(\ty_j = 1 \given \bx) = q_j \teta_j(\bx) \,.
\end{equation}

Therefore, we can appropriately adjust inference procedures of algorithms estimating $\teta_j(\bx)$ to act optimally under different propensity-scored loss functions.

\section{Bayes optimal decisions for Propensity-scored Precision@k}
\label{sec:bayes-optimal-decision}

\citet{Jain_et_al_2016} introduced propensity-scored variants of popular XMLC measures. For precision$@k$ it takes the form:
\begin{equation}
psp@k(\tilde \by, \bh_{@k}(\bx)) = \frac{1}{k} \sum_{j \in \hat \calL_{\bx}} q_j \assert{\tilde{y}_j = 1} \,,
\end{equation}
where $\hat \calL_{\bx}$ is a set of $k$ labels predicted by $\bh_{@k}$ for $\bx$.
Notice that precision$@k$ ($p@k$) is a special case of $psp@k$ if $q_j = 1$ for all $j$.

We define a loss function for propensity-scored precision@$k$ as $\ell_{psp@k} = - psp@k$. 
The conditional risk for $\ell_{psp@k}$ is then:
\begin{eqnarray*}
\loss_{psp@k}(\bh_{@k} \given \bx) & = & \mathbb{E}_{\tby} \ell_{psp@k}(\tby,\bh_{@k}(\bx)) \\
& = & - \sum_{\tby \in \calY} \prob(\tby \given \bx) \frac{1}{k} \sum_{j \in \hat \calL_{\bx}} q_j \assert{\tilde{y}_j = 1} \\
& = & - \frac{1}{k} \sum_{j \in \hat \calL_{\bx}} q_j \sum_{\tby \in \calY} \prob(\tby \given \bx) \assert{\tilde{y}_j = 1} \\
& = & - \frac{1}{k} \sum_{j \in \hat \calL_{\bx}} q_j \teta_j(\bx) \,.
\end{eqnarray*}

The above result shows that the Bayes optimal classifier for $psp@k$ is determined by the 
conditional probabilities of labels scaled by the inverse of the label propensity. Given that the propensities or their estimates are given in the time of prediction, $psp@k$ is optimized by selecting $k$ labels with the highest values of $q_j \teta_j(\bx)$. 


\section{Propensity-scored Probabilistic label tress}

Conditional probabilities of labels can be estimated using many types of multi-label classifiers, such as decision trees, k-nearest neighbors, or binary relevance (\Algo{BR}) trained with proper composite surrogate losses, e.g., squared error, squared hinge, logistic or exponential loss~\citep{Zhang_2004, Agarwal_2014}.
For such models, where estimates of $\teta_j(\bx)$ are available for all $j \in \labels$, application of the Bayes decision rule for propensity-scored measures is straightforward. However, in many XMLC applications, calculating the full set of 
conditional probabilities is not feasible. 
In this section, we introduce an algorithmic solution of applying the Bayes decision rule for $psp@k$ to probabilistic label trees (\Algo{PLT}s).

\subsection{Probabilistic labels trees (\Algo{PLT}s)}

We denote a tree by $\tree$, a set of all its nodes by $\nodes_T$, a root node by $\root_T$, and the set of leaves by $\leaves_T$. The leaf $\leafnode_j \in \leaves_T$ corresponds to the label $j \in \labels$. The parent node of $v$ is denoted by $\pa{\node}$, and the set of child nodes by $\childs{\node}$. The set of leaves of a (sub)tree rooted in node $v$ is denoted by $\leaves_v$, and path from node $v$ to the root by $\Path{\node}$. 

A \Algo{PLT} uses tree $T$ to factorize 
conditional probabilities of labels, $\eta_j(x) = \prob(y_j = 1 \vert \bx)$, $j \in \calL$, by using the chain rule. Let us define an event that $\calL_{\bx}$ contains at least one relevant label in $\leaves_{\node}$: $z_v = (|\{j : \leafnode_j \in \leaves_v\} \cap \labels_{\bx}| > 0)$. Now for every node $v \in \nodes_T$, 
the conditional probability of containing at least one relevant label is given by:
\begin{equation}
\prob(z_v = 1|\bx) = \eta_v(\bx) = \prod_{v' \in \Path{v}} \eta(\bx, v') \,,
\label{eqn:plt-factorization-prediction}
\end{equation}
where $\eta(\bx, v) = \prob(z_v = 1 | z_{\pa{v}} = 1, \bx)$ for non-root nodes, and $\eta(\bx, v) = \prob(z_v = 1 \given \bx)$ for the root. Notice that (\ref{eqn:plt-factorization-prediction}) can also be stated as recursion:
\begin{equation}
\eta_v(\bx) = \eta(\bx, v) \eta_{\pa{v}}(\bx) \,,
\label{eqn:plt-estimates-factorization-recursion}
\end{equation}
and that for leaf nodes we get the 
conditional probabilities of labels: 
\begin{equation}
\eta_{\leafnode_j}(\bx) = \eta_j(\bx) \,, \quad \textrm{for~} l_j \in L_T \,.
\label{eqn:plt_leaf_prob}
\end{equation}

To obtain a \Algo{PLT}, it suffices for a given $T$ to train probabilistic classifiers from $\calH : \R^d \mapsto [0,1]$, 
estimating $\eta(\bx, v)$ for all $v \in V_T$. We denote estimates of $\eta$ by $\heta$.
We index this set of classifiers by the elements of $V_T$ as $H = \{ \heta(v) \in \calH : v \in V_T \}$.

\subsection{Plug-in Bayes optimal prediction \Algo{PLT}s}

\input{figs/algo-psplt}

An inference procedure for \Algo{PLT}s, based on \Algo{uniform-cost search}, has been introduced in \citep{Jasinska_et_al_2016}. It efficiently finds $k$ leaves, with highest $\heta_j(\bx)$ values. 
Since inverse propensity is larger than one, the same method cannot be reliably applied to find leaves with the $k$ highest products of $q_j$ and $\hteta_j(\bx)$. 
To do it, we modify this procedure to an \Algo{$A^*$-search}-style algorithm. 
To this end we introduce cost function $f(\leafnode_j, \bx)$ for each path from the root to a leaf.
Notice that:
\begin{equation}
q_j\hteta_j(\bx) = \exp \Bigg(-\bigg(- \log q_j - \sum_{v \in \Path{l_j}} \log \hteta(\bx, v) \bigg)\Bigg) \,.
\end{equation}
This allows us to use the following definition of the cost function: 
\begin{equation}
f(\leafnode_j, \bx) = \log q_{\max} - \log q_j - \sum_{v \in \Path{l_j}} \log \hteta(\bx, v) \,,
\end{equation}
where $q_{\max} = \max_{j \in \labels} q_j$ is a natural upper bound of $q_j \hteta_j(\bx)$ for all paths. 
We can then guide the \Algo{A*-search} with function $\hat f(v, \bx) = g(v, \bx) + h(v, \bx)$, 
estimating the value of the optimal path, where: 
\begin{equation}
g(v, \bx) = - \sum_{v' \in \Path{v}} \log \hteta(\bx, v')
\end{equation}
is a cost of reaching tree node $v$ from the root, and: 
\begin{equation}
h(v, \bx) =  \log q_{\max} -\log \max_{j \in \labels_{v}}q_j
\end{equation}
is a heuristic function estimating the cost of reaching the best leaf node from node $v$.
To guarantee that \Algo{$A^*$-search} finds the optimal solution---top-$k$ labels with the highest  $f(\leafnode_j, \bx)$ and thereby top-$k$ labels with the highest $q_j\hteta_j(\bx)$---%
we need to ensure that $h(v, \bx)$ is admissible, i.e., it never overestimates the cost of reaching a leaf node~\citep{Russell_Norvig_2016}.
We also would like $h(v, \bx)$ to be consistent, making the \Algo{$A^*$-search} optimally efficient, i.e., there is no other algorithm used with the heuristic that expands fewer nodes~\citep{Russell_Norvig_2016}.
Notice that the heuristic function assumes that probabilities estimated in nodes in a subtree rooted in $v$ are equal to 1. Since $\log 1 = 0$, the heuristic comes to finding the label in the subtree of $v$ with the largest value of the inverse propensity.

Algorithm~\ref{alg:ps-plt-prediction} outlines the prediction procedure for \Algo{PLT}s 
that returns the top-$k$ labels with the highest values of $q_j\hteta_j(\bx)$. 
We call this algorithm Propensity-scored PLTs (\Algo{PS-PLT}s). 
The algorithm is very similar to the original \Algo{Uniform-Cost Search} prediction procedure 
used in \Algo{PLT}s, 
which finds the top-$k$ labels with the highest $\heta_j(\bx)$. 
The difference is that nodes in \Algo{PS-PLT} are evaluated in the ascending order 
of their estimated cost values $\hat f(v, \bx)$ 
instead of decreasing conditional probabilities $\heta_v(\bx)$.

\begin{restatable}{theorem}{optimal-efficiency-of-psplt}
\label{thm:optimal-efficiency-of-psplt}
For any $T, H, \bq$, and $\bx$ the Algorithm~\ref{alg:ps-plt-prediction} is admissible
and optimally efficient.
\end{restatable}

\begin{proof}
\Algo{$A^*$-search} finds an optimal solution if the heuristic $h$ is admissible, i.e., 
if it never overestimates the true value of $h^*$, 
the cost value of reaching the best leaf in a subtree of node $v$.
For node $\node \in \nodes$, we have:
\begin{equation}
h^*(v, \bx) = \log q_{\max} - \log \max_{j \in \labels_{v}} q_j - \!\!\!\sum_{v'\in \Path{l_j}\setminus\Path{v} } \!\!\! \log \hteta(\bx, v') \,.
\end{equation}
Since $\hteta(\bx, v) \in [0, 1]$ and therefore $\log \hteta(\bx, v) \le 0$, 
we have that $h^*(v, \bx) \ge h(v, \bx)$, for all $v \in V_T$, which proves admissibility. 

\Algo{$A^*$-search} is optimally efficient if $h(v, \bx)$ is consistent (monotone), i.e., its estimate is always less than or equal to the estimate for any child node plus the cost of reaching that child. 
Since we have that $\max_{j \in \leaves_{\pa{v}}} q_j \ge \max_{j \in \leaves_{v}} q_j$, and the cost of reaching $v$ from $\pa{v}$ is $-\log(\hteta(\bx, v))$ which is greater or equal 0,
it holds that $h(\pa{v}, \bx) \le h(v, \bx) - \log(\hteta(\bx, v))$.
\end{proof}

The same cost function $f(\leafnode_j, \bx)$ can be used with other tree inference algorithms (for example discussed by \citet{Jasinska-Kobus_et_al_2020}), including \Algo{beam search}~\citep{Kumar_et_al_2013}, that is approximate method for finding $k$ leaves with highest $\heta_j(\bx)$. It is used in many existing label tree implementations such as \Algo{Parabel}, \Algo{Bonsai}, \Algo{AttentionXML} and \Algo{PECOS}. We present \Algo{beam search} variant of \Algo{PS-PLT} in the Appendix.


\section{Experimental results}
\label{sec:experimental-results}

\input{figs/table-psplt-results}
\input{figs/table-psplt-times}

In this section, we empirically show the usefulness of the proposed plug-in approach by incorporating it into \Algo{BR} and \Algo{PLT} algorithms and comparing these algorithms to their vanilla versions and state-of-the-art methods, particularly those that focus on tail-labels performance: \Algo{PFastreXML}~\citep{Jain_et_al_2016}, \Algo{ProXML}~\citep{Babbar_Scholkopf_2019}, 
a variant of \Algo{DiSMEC}~\citep{Babbar_Scholkopf_2017} with a re-balanced and unbiased loss function as implemented in \Algo{PW-DiSMEC}~\citep{Qaraei_et_al_2021} (class-balanced variant),  
and \Algo{Parabel}~\citep{Prabhu_et_al_2018}. We conduct a comparison on six well-established XMLC benchmark datasets from the XMLC repository~\citep{Bhatia_et_al_2016}, for which we use the original train and test splits. Statistics of the used datasets can be found in the Appendix. For algorithms listed above, we report results as found in respective papers.

Since true propensities are unknown for the benchmark datasets, as true $\by$ is unavailable due to the large label space, 
for empirical evaluation we model propensities as proposed by \citet{Jain_et_al_2016}: 

\begin{equation}
    p_j = \prob(\tilde{y}_j= 1 \given y_{j} = 1) = \frac{1}{q_j} = \frac{1}{1 + C e^{-A \log (N_j + B)}} \,,
\end{equation}
where $N_j$ is the number of data points annotated with label $j$ in the observed ground truth dataset of size $N$, parameters $A$ and $B$ are specific for each dataset, and $C = (\log N - 1)(B + 1)^A$. 
We calculate propensity values on train set for each dataset using parameter values recommended in \citep{Jain_et_al_2016}. Values of $A$ and $B$ are included in Table~\ref{tab:psplt-vs-sota}. We evaluate all algorithms with both propensity-scored and standard precision$@k$.

We modified the recently introduced \Algo{napkinXC}~\citep{Jasinska-Kobus_et_al_2020} implementation of \Algo{PLT}s,
\footnote{Repository with the code and scripts to reproduce the experiments: \url{https://github.com/mwydmuch/napkinXC}} 
which obtains state-of-the-art results and uses the \Algo{Uniform-Cost Search} as its inference method.
We train binary models in both \Algo{BR} and \Algo{PLT}s using the \Algo{LIBLINEAR} library~\citep{liblinear} with $L_2$-regularized logistic regression. For \Algo{PLT}s, we use an ensemble of 3 trees built with the hierarchical 2-means clustering algorithm (with clusters of size 100), popularized by \Algo{Parabel}~\citep{Prabhu_et_al_2018}. Because the tree-building procedure involves randomness, we repeat all \Algo{PLT}s experiments five times and report the mean performance. We report standard errors along with additional results for popular $L_2$-regularized squared hinge loss and for \Algo{beam search} variant of \Algo{PS-PLT} in the Appendix. The experiments were performed on an Intel Xeon E5-2697 v3 2.6GHz machine with 128GB of memory. 

The main results of the experimental comparison are presented in Table~\ref{tab:psplt-vs-sota}. 
Propensity-scored \Algo{BR} and \Algo{PLT}s consistently obtain better propensity-scored precision$@k$. 
At the same time, they slightly drop the performance on the standard precision$@k$ on four and improve it on two datasets. There is no single method that dominates others on all datasets, but \Algo{PS-PLT}s is the best sub-linear method, achieving best results on $psp@\{1,3,5\}$ in this category on five out of six datasets, at the same time in many cases being competitive to \Algo{ProXML} and \Algo{PW-DiSMEC} that often require orders of magnitude more time for training and prediction than \Algo{PS-PLT}. In Table~\ref{tab:psplt-vs-plt-test-time}, we show CPU train and test times of \Algo{PS-PLT}s compared to vanilla \Algo{PLT}s, \Algo{PfasterXML}, \Algo{ProXML} and \Algo{PW-DiSMEC} on our hardware (approximated for the last two using a subset of labels).

\section{Conclusions}

In this work, we demonstrated a simple approach for obtaining Bayes optimal predictions for propensity-scored precision$@k$, which can be applied to a wide group of probabilistic classifiers. Particularly we introduced an admissible and consistent inference algorithm for probabilistic labels trees, being the underlying model of such methods
like \Algo{Parabel}, \Algo{Bonsai}, \Algo{napkinXC}, \Algo{extremeText}, \Algo{AttentionXML} and \Algo{PECOS}.

\Algo{PS-PLT}s show significant improvement with respect to propensity-scored precision$@k$, achieving state-of-the-art results in the group of algorithms with sub-linear training and prediction times. Furthermore, the introduced approach does not require any retraining of underlining classifiers if the propensities change. Since in real-world applications estimating true propensities may be hard, this property makes our approach suitable for dynamically changing environments, especially if we take into account the fact that many of \Algo{PLT}s-based algorithms can be trained incrementally~\citep{Jasinska_et_al_2016,Wydmuch_et_al_2018,You_et_al_2019,Jasinska-Kobus_et_al_2021}.

\section*{Acknowledgments}

Computational experiments have been performed in Poznan Supercomputing and Networking Center.

\bibliographystyle{ACM-Reference-Format}
\balance
\bibliography{references}

\appendix
\onecolumn

\pagebreak

\section{Datasets}

\input{figs/table-datasets}

\section{PS-PLT with \Algo{beam search} inference}

\Algo{Beam search} is a greedy search method that on each level of the tree keeps only $b$ nodes with the highest probability estimates and discards the rest. Therefore it may not find the actual top $k$ labels and may suffer regret for precision$@k$~\citep{Zhuo_et_al_2020}, but it guarantees logarithmic time and performs prediction level-by-level, which allows for easier implementation and memory management in large models. In Algorithm~\ref{alg:ps-plt-beam-search-prediction} we present \Algo{beam search} variant of \Algo{PS-PLT}. The presented algorithm assumes that tree $T$ is balanced. 

\input{figs/algo-psplt-beam-search}

\section{Detailed results of different variants of PS-PLTs}

In Table~\ref{tab:psplt-details} we report the detailed results of \Algo{PLT} with nodes trained using logistic loss ($log$) and squared hinge loss ($h^2$) and \Algo{PS-PLT} with \Algo{A*-search} ($A^*$) presented in Algorithm~\ref{alg:ps-plt-prediction} as well as with \Algo{beam search} version ($beam$) presented in Algorithm~\ref{alg:ps-plt-beam-search-prediction}. For \Algo{beam search} variant we use $b = 10$ which is default value in many popular implementations, since it provides good trade off between predictive and computational performance when predicting top-5 labels. 
All variants use the ensemble of 3 trees and the same tree structures, built with the hierarchical 2-means clustering algorithm (with clusters of size 100). This means that the difference between variants is only in learning node classifiers and inference (tree search) methods.

The results show that all variants of \Algo{PL-PLT}s consistently obtain better propensity-scored precision$@k$. \Algo{PS-PLT}s trained with logistic loss achieves greater improvement in terms of $psp@\{1,3,5\}$ over vanilla \Algo{PLT} than variant trained with squared hinge loss. While \Algo{PS-PLT}s trained with squared hinge loss suffer a small drop in the performance on the standard precision$@k$. For both losses, \Algo{beam search} variant allows for further decrease of inference time at the cost of an only minor decrease in terms of predictive performance.

\input{figs/table-psplt-results-ext}

\end{document}

%% file: commands.tex
\usepackage[utf8]{inputenc} 
\usepackage[T1]{fontenc}    
\usepackage[english]{babel}

\PassOptionsToPackage{linktocpage}{hyperref}
\usepackage{url}            
\usepackage{booktabs}       
\usepackage{microtype}      
\usepackage{siunitx}

\usepackage{amsthm}
\usepackage{amsmath}
\usepackage{array,graphicx}

\usepackage{algorithm}
\usepackage[noend]{algpseudocode}

\usepackage{mathtools} 
\usepackage{stmaryrd}
\usepackage{multirow}
\usepackage{makecell}

\usepackage{thmtools} 
\usepackage{thm-restate}
\usepackage{balance}

\usepackage{adjustbox}

\makeatletter

\newcommand*{\eifstartswith}{\@expandtwoargs\ifstartswith}
\newcommand*{\ifstartswith}[2]{%
  \if\@car#1.\@nil\@car#2.\@nil
    \expandafter\@firstoftwo
  \else
    \expandafter\@secondoftwo
  \fi}
\theoremstyle{definition}

\newcommand{\Algo}[1]{\textsc{#1}}

\renewcommand{\vec}[1]{\boldsymbol{#1}}

\newcommand{\bx}{\vec{x}}
\newcommand{\by}{\vec{y}}
\newcommand{\bq}{\vec{q}}
\newcommand{\tby}{\tilde{\vec{y}}}

\newcommand{\bh}{\vec{h}}

\newcommand{\ty}{\tilde{y}}

\newcommand{\calX}{\mathcal{X}}
\newcommand{\calY}{\mathcal{Y}}

\newcommand{\calQ}{\mathcal{Q}}
\newcommand{\calL}{\mathcal{L}}
\newcommand{\calH}{\mathcal{H}}

\newcommand{\calB}{\mathcal{B}}

\newcommand{\heta}{\hat{\eta}}
\newcommand{\teta}{\tilde{\eta}}
\newcommand{\hteta}{\hat{\tilde{\eta}}}
\newcommand{\hy}{\hat{y}}

\newcommand\R{\mathbb{R}}   

\newcommand{\pa}[1]{\mathrm{pa}(#1)}

\newcommand{\Path}[1]{\mathrm{Path}(#1)}

\newcommand{\prob}{\mathbf{P}}

\newcommand{\loss}{L}

\newcommand{\childs}[1]{\mathrm{Ch}(#1)}

\renewcommand{\root}{r}
\newcommand{\nodes}{V}

\newcommand{\tree}{T}
\newcommand{\leaves}{L}
\newcommand{\leafnode}{l}
\newcommand{\labels}{\calL}

\newcommand{\node}{v}

\newcommand{\eurlex}{EurLex-4K}
\newcommand{\amazoncatsmall}{AmazonCat-13K}

\newcommand{\wikiten}{Wiki10-31K}

\newcommand{\wikilshtc}{WikiLSHTC-325K}
\newcommand{\wikipedia}{WikipediaLarge-500K}
\newcommand{\amazon}{Amazon-670K}

\algnewcommand{\IIf}[1]{\State\algorithmicif\ #1\ \algorithmicthen}
\algnewcommand{\IElse}[1]{\State\algorithmicelse\ #1\ }
\algnewcommand{\IElseIf}[1]{\State\algorithmicelse \algorithmicif\ #1\  \algorithmicthen}
\algnewcommand{\EndIIf}{\unskip\ \algorithmicend\ \algorithmicif}
\algnewcommand{\IfThen}[2]{\State\algorithmicif\ #1\ \algorithmicthen\ #2\ }
\algnewcommand{\ForDo}[2]{\State\algorithmicfor\ #1\ \algorithmicdo\ #2\ }

\newcommand{\assert}[1]{\llbracket #1 \rrbracket}

\newcommand{\given}{\, | \,}

\DeclareMathOperator*{\argmin}{\arg \min}



%% file: figs/algo-psplt.tex
\begin{algorithm*}[!ht]
\caption{\Algo{PS-PLT.PredictTopLabels}$(T, H, \bq, \bx, k)$}
\label{alg:ps-plt-prediction}
\begin{small}
\begin{algorithmic}[1] 
\State $\hat\by = \vec{0}$, $q_{\max} = \max_{j \in \labels} q_j$, $\calQ = \emptyset$,  \Comment{Initialize prediction $\hat\by$ vector to all zeros, $q_{\max}$ and a priority queue $\calQ$, ordered ascending by $\hat f(v, \bx)$}

\State $g(r_T, \bx) = - \log \hteta(\bx,r_T)$ \Comment{Calculate cost $g(r_T, \bx)$ for the tree}
\State $\hat f(r_T, \bx) = g(r_T, \bx) + \log q_{\max} -\log \max_{j \in \labels_{r_T}}q_j$ \Comment{Calculate estimated cost $\hat f(r_T, \bx)$ for the tree root}

\State $\calQ\mathrm{.add}((r_T, g(r_T, \bx), \hat f(r_T, \bx))$ \Comment{Add the tree root with cost $g(r_T, \bx)$ and estimation $\hat f(r_T, \bx)$ to the queue}
\While{$\|\hat\by\|_1 < k$} \Comment{While the number of predicted labels is less than $k$}
	\State $(v, g(v, \bx),\_) = \calQ\mathrm{.pop}()$ \Comment{Pop the  element with the lowest cost from the queue (only node and corresponding probability)}
	\If{$v$ is a leaf} 
		$\hy_v = 1$ \Comment{If the node is a leaf, set the corresponding label in the prediction vector}
	\Else $\,$ \textbf{for} $v' \in \childs{v}$ \textbf{do} \Comment{If the node is an internal node, for all child nodes}
	    \State $g(v', \bx) = g(v, \bx) - \log \hteta(\bx,v')$ \Comment{Compute $g(v', \bx)$ using $\hteta(v', \bx) \in H$}
	    \State $\hat f(v',\bx) = g(v', \bx) + \log q_{\max} -\log \max_{j \in \labels_{v'}}q_j$  \Comment{Calculate estimation $\hat f(v',\bx)$}
		\State $\calQ\mathrm{.add}((v', g(v', \bx), \hat f(v',\bx)))$  \Comment{Add the node, computed cost $g(v', \bx)$, and estimation $\hat f(v',\bx)$ to the queue}
	\EndIf 
\EndWhile
\State \textbf{return} $\hat\by$ \Comment{Return the prediction vector}
\end{algorithmic}
\end{small}
\end{algorithm*}

%% file: figs/table-psplt-results.tex
\begin{table*}[!ht]

\begin{center}
\caption{\Algo{PS-PLT}s and \Algo{PLT}s compared to other state-of-the-art algorithms on propensity-scored and standard  precision$@\{1, 3, 5\}$~$[\%]$. The best result for each measure is in bold. The best result in the group of sub-linear methods (the last 4 methods) is underlined.}
\label{tab:psplt-vs-sota}

\tabcolsep=3pt
\resizebox{\textwidth}{!}{
\begin{tabular}{l|rrr|rrr|rrr|rrr|rrr|rrr}

\toprule
\multicolumn{1}{c|}{Algorithm} 
& \multicolumn{1}{c}{$psp@1$} 
& \multicolumn{1}{c}{$psp@3$} 
& \multicolumn{1}{c|}{$psp@5$}
& \multicolumn{1}{c}{$p@1$} 
& \multicolumn{1}{c}{$p@3$} 
& \multicolumn{1}{c|}{$p@5$}
& \multicolumn{1}{c}{$psp@1$} 
& \multicolumn{1}{c}{$psp@3$} 
& \multicolumn{1}{c|}{$psp@5$}
& \multicolumn{1}{c}{$p@1$} 
& \multicolumn{1}{c}{$p@3$} 
& \multicolumn{1}{c|}{$p@5$}
& \multicolumn{1}{c}{$psp@1$} 
& \multicolumn{1}{c}{$psp@3$} 
& \multicolumn{1}{c|}{$psp@5$}
& \multicolumn{1}{c}{$p@1$} 
& \multicolumn{1}{c}{$p@3$} 
& \multicolumn{1}{c}{$p@5$} \\

\specialrule{0.70pt}{0.4ex}{0.65ex}
& \multicolumn{6}{c|}{\eurlex, $A=0.55, B=1.5$} & \multicolumn{6}{c|}{\amazoncatsmall, $A=0.55, B=1.5$} & \multicolumn{6}{c}{\wikiten, $A=0.55, B=1.5$} \\
\midrule
\Algo{ProXML} & 45.20 & 48.50 & 51.00 & \textbf{86.50} & 68.40 & 53.20 & \multicolumn{6}{c|}{results not reported} & \multicolumn{6}{c}{results not reported} \\
\Algo{PW-DiSMEC} & 43.48 & 48.81 & 51.25 & 82.25 & 68.80 & 57.18 & 64.95 & 71.35 & 74.37 & \textbf{93.54} & 78.50 & 63.33 & 12.67 & 15.87 & 18.28 & \textbf{85.77} & \textbf{78.17} & \textbf{68.53} \\
\Algo{BR} & 36.67 & 44.54 & 49.05 & 81.91 & \textbf{68.85} & \textbf{57.83} & 51.54 & 64.16 & 71.20 & 92.89 & 78.35 & 63.69 & 12.03 & 13.24 & 14.07 & 84.49 & 72.50 & 63.23 \\
\Algo{PS-BR} & \textbf{46.13} & \textbf{49.60} & \textbf{51.78} & 78.45 & 68.01 & 57.62 & 66.00 & 71.28 & 74.08 & 86.55 & 76.22 & 63.15 & 19.24 & 17.69 & 17.60 & 80.61 & 69.70 & 61.86 \\
\midrule
\Algo{PfastreXML} & 43.86 & 45.72 & 46.97 & 75.45 & 62.70 & 52.51 & \underline{\textbf{69.52}} & \underline{\textbf{73.22}} & \underline{\textbf{75.48}} & 91.75 & 77.97 & 63.68 & 19.02 & 18.34 & 18.43 & 83.57 & 68.61 & 59.10 \\
\Algo{Parabel} & 36.36 & 44.04 & 48.29 & 81.73 & \underline{68.78} & \underline{57.44} & 50.93 & 64.00 & 72.08 & 93.03 & \underline{\textbf{79.16}} & \underline{\textbf{64.52}} & 11.66 & 12.73 & 13.68 & 84.31 & 72.57 & 63.39 \\
\Algo{PLT} & 36.00 & 43.30 & 47.31 & \underline{81.77} & 68.33 & 57.15 & 50.02 & 63.15 & 71.24 & \underline{\textbf{93.37}} & 78.90 & 64.18 & 12.77 & 14.45 & 15.12 & \underline{85.54} & \underline{74.56} & \underline{64.48} \\
\Algo{PS-PLT} & \underline{44.73} & \underline{48.52} & \underline{50.84} & 79.19 & 67.81 & 57.15 & 66.81 & 72.05 & 74.88 & 88.04 & 77.16 & 63.84 & \underline{\textbf{21.83}} & \underline{\textbf{19.77}} & \underline{\textbf{19.12}} & 74.12 & 65.87 & 59.08 \\

\specialrule{0.70pt}{0.4ex}{0.65ex}
& \multicolumn{6}{c|}{\wikilshtc, $A=0.5, B=0.4$} & \multicolumn{6}{c|}{\wikipedia, $A=0.5, B=0.4$} & \multicolumn{6}{c}{\amazon, $A=0.6, B=2.6$} \\
\midrule
\Algo{ProXML} & 34.80 & 37.70 & 41.00 & 63.60 & 41.50 & 30.80 & 33.10 & 35.00 & \textbf{39.40} & \textbf{68.80} & \textbf{48.90} & 37.90 & 30.80 & 32.80 & 35.10 & 43.50 & 38.70 & 35.30 \\
\Algo{PW-DiSMEC} & \textbf{37.12} & \textbf{40.36} & \textbf{43.57} & \textbf{65.27} & 42.68 & 31.48 & 30.32 & 31.56 & 33.52 & 66.38 & 45.69 & 35.85 & \textbf{31.24} & 33.27 & 35.51 & 41.70 & 37.81 & 34.92 \\
\midrule
\Algo{PfastreXML} & 30.66 & 31.55 & 33.12 & 56.05 & 36.79 & 27.09 & 29.20 & 27.60 & 27.70 & 59.50 & 40.20 & 30.70 & 29.30 & 30.80 & 32.43 & 39.46 & 35.81 & 33.05 \\
\Algo{Parabel} & 26.76 & 33.27 & 37.36 & \underline{65.04} & \underline{\textbf{43.23}} & \underline{\textbf{32.05}} & 28.80 & 31.90 & 34.60 & 67.50 & 48.70 & 37.70 & 25.43 & 29.43 & 32.85 & 44.89 & 39.80 & 36.00 \\
\Algo{PLT} & 26.00 & 31.93 & 35.62 & 63.87 & 42.25 & 31.34 & 26.28 & 30.93 & 34.15 & 67.50 & 48.26 & 37.74 & 26.31 & 30.22 & 33.83 & \underline{\textbf{45.01}} & \underline{\textbf{40.21}} & \underline{\textbf{36.72}} \\
\Algo{PS-PLT} & \underline{32.84} & \underline{36.17} & \underline{39.20} & 64.57 & 43.17 & 32.01 & \underline{\textbf{34.12}} & \underline{\textbf{35.70}} & \underline{38.14} & \underline{67.53} & \underline{48.68} & \underline{\textbf{38.23}} & \underline{31.14} & \underline{\textbf{33.45}} & \underline{\textbf{35.60}} & 43.71 & 39.72 & 36.60 \\

\bottomrule
\end{tabular}
}
\end{center}
\end{table*}

%% file: figs/table-psplt-times.tex
\begin{table}[!ht]
\begin{center}

\caption{\Algo{PS-PLT} and \Algo{PLT} average CPU train and prediction time compared to other state-of-the-art algorithms.}
\label{tab:psplt-vs-plt-test-time}

\tabcolsep=3pt
\resizebox{\linewidth}{!}{
\begin{tabular}{l|rrr|rr}

\toprule
\multicolumn{1}{c|}{Dataset}
& \multicolumn{1}{c}{\Algo{ProXML}}
& \multicolumn{1}{c}{\Algo{PW-DiSMEC}}
& \multicolumn{1}{c|}{\Algo{PfastreXML}}
& \multicolumn{1}{c}{\Algo{PLT}}
& \multicolumn{1}{c}{\Algo{PS-PLT}} \\
\specialrule{0.70pt}{0.4ex}{0.65ex}
& \multicolumn{5}{c}{$t_{train}$ [h]} \\
\midrule
\wikilshtc & $\approx151760$ & $\approx1437$ & 6.25 & \multicolumn{2}{S}{9.21} \\
\wikipedia & $\approx1595920$ & $\approx16272$ & 51.07 & \multicolumn{2}{S}{46.17} \\
\amazon & $\approx75160$ & $\approx 810$ & 3.01 & \multicolumn{2}{S}{1.92} \\
\specialrule{0.70pt}{0.4ex}{0.65ex}
& \multicolumn{5}{c}{$t_{test}/N_{test}$ [ms]} \\
\midrule

\wikilshtc & $\approx90$ & $\approx82$ & 4.10 & 4.96 & 12.40 \\
\wikipedia & $\approx496$ & $\approx457$ & 15.24 & 26.40 & 60.01 \\
\amazon & $\approx111$ & $\approx 103$ & 9.96 & 12.06 & 20.40 \\

\bottomrule
\end{tabular}
}
\end{center}
\end{table}

%% file: figs/table-datasets.tex
\begin{table*}[!h]
    \centering
    \caption{The number of unique features, labels, examples in train and test splits, and the average number of true labels per example in the benchmark data sets and corresponding $A$, $B$ parameters for empirical propensity modeling.}
    \label{tab:datasets}
    
    \begin{tabular}{l|rrrrr|ll}
        \toprule
        Dataset & \multicolumn{1}{c}{$\dim{\calX}$} & \multicolumn{1}{c}{$\dim{\calY}$ ($m$)} &  \multicolumn{1}{c}{$N_{\textrm{train}}$} & \multicolumn{1}{c}{$N_{\textrm{test}}$} & \multicolumn{1}{c|}{avg. $|\labels_{\bx}|$} & \multicolumn{1}{c}{$A$} & \multicolumn{1}{c}{$B$} \\
        \midrule
        \eurlex         & 5000      & 3993      & 15539      & 3809      & 5.31  & 0.55 & 1.5 \\
        \amazoncatsmall & 203882    & 13330     & 1186239    & 306782    & 5.04  & 0.55 & 1.5 \\
        \wikiten        & 101938    & 30938     & 14146      & 6616      & 18.64 & 0.55 & 1.5 \\
        \wikilshtc      & 1617899   & 325056    & 1778351    & 587084    & 3.19  & 0.5  & 0.4 \\
        \wikipedia      & 2381304   & 501070    & 1813391    & 783743    & 4.77  & 0.5 & 0.4 \\
        \amazon         & 135909    & 670091    & 490449     & 153025    & 5.45  & 0.6  & 2.6 \\
        \bottomrule
    \end{tabular}

\end{table*}

%% file: figs/algo-psplt-beam-search.tex
\begin{algorithm*}[!ht]
\caption{\Algo{PS-PLT.PredictTopLabelsWithBeamSearch}$(T, H, \bq, \bx, k, b)$}
\label{alg:ps-plt-beam-search-prediction}
\begin{small}
\begin{algorithmic}[1] 
\State $\hat\by = \vec{0}$, $q_{\max} = \max_{j \in \labels} q_j$, $\calB = \emptyset$,  \Comment{Initialize prediction $\hat\by$ vector to all zeros, $q_{\max}$ and a list $\calB$}
\State $g(r_T, \bx) = - \log \hteta(\bx,r_T)$ \Comment{Calculate cost $g(r_T, \bx)$ for the tree}
\State $\hat f(r_T, \bx) = g(r_T, \bx) + \log q_{\max} -\log \max_{j \in \labels_{r_T}}q_j$ \Comment{Calculate estimated cost $\hat f(r_T, \bx)$ for the tree root}
\State $\calB\mathrm{.add}((r_T, g(r_T, \bx), \hat f(r_T, \bx))$ \Comment{Add the tree root with cost $g(r_T, \bx)$ and estimation $\hat f(r_T, \bx)$ to the list}
\For{$d = 0$; $d < $ depth of $T$; $d = d + 1$} \Comment{For each level of the tree $T$}
	\State{$\calB' = $ \Algo{SelectTopNodes}($\calB$, $b$)} \Comment{Select $b$ nodes from $\calB'$ with highest values of $\hat f(v',\bx)$}
    \State{$\calB = \emptyset$} \Comment{Initialize list of nodes of the next level of the tree}
    \For{$(v, g(v, \bx),\_) \in \calB'$} \Comment{Iterate over elements on the list $\calB'$ (nodes and corresponding probabilities)}
    	\For{$v' \in \childs{v}$} \Comment{For all child nodes}
    	    \State $g(v', \bx) = g(v, \bx) - \log \hteta(\bx,v')$ \Comment{Compute $g(v', \bx)$ using $\hteta(v', \bx) \in H$}
    	    \State $\hat f(v',\bx) = g(v', \bx) + \log q_{\max} -\log \max_{j \in \labels_{v'}}q_j$  \Comment{Calculate estimation $\hat f(v',\bx)$}
    		\State $\calB\mathrm{.add}((v', g(v', \bx), \hat f(v',\bx)))$  \Comment{Add the node, computed cost $g(v', \bx)$, and estimation $\hat f(v',\bx)$ to the list $\calB$}
    	\EndFor	
	\EndFor
\EndFor
\For{($v$, \_, \_) $\in$ \Algo{SelectTopNodes}($\calB$, $k$)} $\hy_v = 1$ \Comment{Select $k$ leaves from $\calB$ with highest values of $\hat f(v',\bx)$ and set the corresponding labels in $\hat\by$}
\EndFor
\State \textbf{return} $\hat\by$ \Comment{Return the prediction vector}
\end{algorithmic}
\end{small}
\end{algorithm*}

%% file: figs/table-psplt-results-ext.tex
\begin{table*}[!h]
\caption{Mean performance with standard errors, rounded to two decimal places, of different variants of \Algo{PS-PLT}s on propensity-scored and standard  precision$@\{1, 3, 5\}$~$[\%]$, train time~$[h]$ and inference time per example~$[ms]$. The best result for each measure is in bold.}
\begin{center}
\label{tab:psplt-details}
\tabcolsep=3pt
\begin{tabular}{l|rrr|rrr|lr}

\toprule
\multicolumn{1}{c|}{} 
& \multicolumn{1}{c}{$psp@1$} 
& \multicolumn{1}{c}{$psp@3$} 
& \multicolumn{1}{c|}{$psp@5$}
& \multicolumn{1}{c}{$p@1$} 
& \multicolumn{1}{c}{$p@3$} 
& \multicolumn{1}{c|}{$p@5$}
& \multicolumn{1}{c}{$T_{train}$}
& \multicolumn{1}{c}{$T_{test}/N_{test}$} \\

\specialrule{0.94pt}{0.4ex}{0.65ex}
& \multicolumn{8}{c}{\eurlex} \\
\midrule
PLT$_{log}$ & 36.00 $\pm$ 0.07 & 43.30 $\pm$ 0.09 & 47.31 $\pm$ 0.09 & \textbf{81.77 $\pm$ 0.09} & 68.33 $\pm$ 0.11 & 57.15 $\pm$ 0.08 & \multirow{3}{*}{$\begin{rcases}\\ \\ \\ \end{rcases}$ 0.04 $\pm$ 0.00} & 2.83 $\pm$ 0.10 \\
PS-PLT$_{log + A^*}$ & \textbf{44.73 $\pm$ 0.06} & \textbf{48.52 $\pm$ 0.11} & \textbf{50.84 $\pm$ 0.11} & 79.19 $\pm$ 0.09 & 67.81 $\pm$ 0.07 & 57.15 $\pm$ 0.09 & & 5.66 $\pm$ 0.14 \\
PS-PLT$_{log + beam}$ & 44.72 $\pm$ 0.07 & 48.48 $\pm$ 0.11 & 50.77 $\pm$ 0.12 & 79.19 $\pm$ 0.09 & 67.78 $\pm$ 0.06 & 57.12 $\pm$ 0.11 & & 1.75 $\pm$ 0.12 \\
PLT$_{h^2}$ & 36.21 $\pm$ 0.05 & 44.01 $\pm$ 0.12 & 48.41 $\pm$ 0.16 & 81.66 $\pm$ 0.14 & \textbf{68.75 $\pm$ 0.15} & \textbf{57.54 $\pm$ 0.14} & \multirow{3}{*}{$\begin{rcases}\\ \\ \\ \end{rcases}$ \textbf{0.02 $\pm$ 0.00}} & 1.83 $\pm$ 0.07 \\
PS-PLT$_{h^2 + A^*}$ & 44.21 $\pm$ 0.05 & 48.51 $\pm$ 0.12 & 50.60 $\pm$ 0.15 & 80.72 $\pm$ 0.14 & 67.99 $\pm$ 0.10 & 56.20 $\pm$ 0.14 & & 2.60 $\pm$ 0.17 \\
PS-PLT$_{h^2 + beam}$ & 44.21 $\pm$ 0.08 & 48.49 $\pm$ 0.13 & 50.57 $\pm$ 0.13 & 80.69 $\pm$ 0.14 & 67.97 $\pm$ 0.12 & 56.22 $\pm$ 0.11 & & \textbf{1.65 $\pm$ 0.02} \\

\specialrule{0.94pt}{0.4ex}{0.65ex}
& \multicolumn{8}{c}{\amazoncatsmall} \\
\midrule
PLT$_{log}$ & 50.02 $\pm$ 0.01 & 63.15 $\pm$ 0.03 & 71.24 $\pm$ 0.06 & \textbf{93.37 $\pm$ 0.02} & 78.90 $\pm$ 0.04 & 64.18 $\pm$ 0.05 & \multirow{3}{*}{$\begin{rcases}\\ \\ \\ \end{rcases}$ 3.14 $\pm$ 0.06} & 1.74 $\pm$ 0.07 \\
PS-PLT$_{log + A^*}$ & \textbf{66.81 $\pm$ 0.03} & \textbf{72.05 $\pm$ 0.04} & \textbf{74.88 $\pm$ 0.05} & 88.04 $\pm$ 0.05 & 77.16 $\pm$ 0.04 & 63.84 $\pm$ 0.03 & & 3.71 $\pm$ 0.37 \\
PS-PLT$_{log + beam}$ & 66.78 $\pm$ 0.03 & 72.01 $\pm$ 0.04 & 74.85 $\pm$ 0.04 & 88.04 $\pm$ 0.05 & 77.16 $\pm$ 0.04 & 63.84 $\pm$ 0.03 & & 1.19 $\pm$ 0.09 \\
PLT$_{h^2}$ & 50.91 $\pm$ 0.01 & 63.89 $\pm$ 0.03 & 71.94 $\pm$ 0.06 & 93.00 $\pm$ 0.04 & \textbf{79.06 $\pm$ 0.04} & \textbf{64.43 $\pm$ 0.04} & \multirow{3}{*}{$\begin{rcases}\\ \\ \\ \end{rcases}$ \textbf{1.01 $\pm$ 0.05}} & 1.20 $\pm$ 0.04 \\
PS-PLT$_{h^2 + A^*}$ & 65.97 $\pm$ 0.03 & 71.96 $\pm$ 0.06 & 74.76 $\pm$ 0.11 & 88.76 $\pm$ 0.05 & 77.75 $\pm$ 0.06 & 63.75 $\pm$ 0.08 & & 2.31 $\pm$ 0.04 \\
PS-PLT$_{h^2 + beam}$ & 65.96 $\pm$ 0.03 & 71.94 $\pm$ 0.07 & 74.84 $\pm$ 0.12 & 88.77 $\pm$ 0.06 & 77.75 $\pm$ 0.06 & 63.84 $\pm$ 0.09 & & \textbf{0.92 $\pm$ 0.03} \\

\specialrule{0.94pt}{0.4ex}{0.65ex}
& \multicolumn{8}{c}{\wikiten} \\
\midrule
PLT$_{log}$ & 12.77 $\pm$ 0.02 & 14.45 $\pm$ 0.01 & 15.12 $\pm$ 0.01 & \textbf{85.54 $\pm$ 0.07} & \textbf{74.56 $\pm$ 0.05} & \textbf{64.48 $\pm$ 0.03} & \multirow{3}{*}{$\begin{rcases}\\ \\ \\ \end{rcases}$ 0.46 $\pm$ 0.01} & 25.08 $\pm$ 0.39 \\
PS-PLT$_{log + A^*}$ & \textbf{21.83 $\pm$ 0.07} & \textbf{19.77 $\pm$ 0.03} & \textbf{19.12 $\pm$ 0.04} & 74.12 $\pm$ 0.09 & 65.87 $\pm$ 0.13 & 59.08 $\pm$ 0.15 & & 74.37 $\pm$ 1.08 \\
PS-PLT$_{log + beam}$ & 21.14 $\pm$ 0.06 & 19.02 $\pm$ 0.05 & 18.43 $\pm$ 0.07 & 74.33 $\pm$ 0.12 & 66.20 $\pm$ 0.23 & 59.62 $\pm$ 0.24 & & 5.63 $\pm$ 0.05 \\
PLT$_{h^2}$ & 11.68 $\pm$ 0.01 & 12.84 $\pm$ 0.02 & 13.79 $\pm$ 0.02 & 84.31 $\pm$ 0.13 & 72.90 $\pm$ 0.06 & 63.75 $\pm$ 0.04 & \multirow{3}{*}{$\begin{rcases}\\ \\ \\ \end{rcases}$ \textbf{0.28 $\pm$ 0.01}} & 10.91 $\pm$ 0.13 \\
PS-PLT$_{h^2 + A^*}$ & 18.51 $\pm$ 0.02 & 17.61 $\pm$ 0.04 & 18.04 $\pm$ 0.05 & 83.06 $\pm$ 0.07 & 71.19 $\pm$ 0.14 & 62.66 $\pm$ 0.11 & & 33.06 $\pm$ 0.70 \\
PS-PLT$_{h^2 + beam}$ & 18.33 $\pm$ 0.02 & 17.36 $\pm$ 0.04 & 17.65 $\pm$ 0.05 & 83.03 $\pm$ 0.07 & 71.17 $\pm$ 0.14 & 62.58 $\pm$ 0.13 & & \textbf{4.81 $\pm$ 0.21} \\

\specialrule{0.94pt}{0.4ex}{0.65ex}
& \multicolumn{8}{c}{\wikilshtc} \\
\midrule
PLT$_{log}$ & 26.00 $\pm$ 0.08 & 31.93 $\pm$ 0.11 & 35.62 $\pm$ 0.13 & 63.87 $\pm$ 0.19 & 42.25 $\pm$ 0.13 & 31.34 $\pm$ 0.10 & \multirow{3}{*}{$\begin{rcases}\\ \\ \\ \end{rcases}$ 9.21 $\pm$ 0.10} & 4.96 $\pm$ 0.16 \\
PS-PLT$_{log + A^*}$ & 32.84 $\pm$ 0.18 & 36.27 $\pm$ 0.24 & \textbf{39.38 $\pm$ 0.27} & 64.47 $\pm$ 0.43 & \textbf{43.19 $\pm$ 0.29} & \textbf{32.08 $\pm$ 0.21} & & 12.40 $\pm$ 0.74 \\
PS-PLT$_{log + beam}$ & 32.76 $\pm$ 0.18 & 36.09 $\pm$ 0.25 & 39.08 $\pm$ 0.31 & 64.38 $\pm$ 0.44 & 43.05 $\pm$ 0.30 & 31.91 $\pm$ 0.23 & & 1.21 $\pm$ 0.04 \\
PLT$_{h^2}$ & 26.71 $\pm$ 0.08 & 33.14 $\pm$ 0.16 & 37.06 $\pm$ 0.22 & 64.69 $\pm$ 0.18 & 42.95 $\pm$ 0.16 & 31.82 $\pm$ 0.15 & \multirow{3}{*}{$\begin{rcases}\\ \\ \\ \end{rcases}$ \textbf{6.35 $\pm$ 0.10}} & 2.31 $\pm$ 0.06 \\
PS-PLT$_{h^2 + A^*}$ & \textbf{33.16 $\pm$ 0.12} & \textbf{36.39 $\pm$ 0.29} & 38.14 $\pm$ 0.42 & \textbf{65.87 $\pm$ 0.23} & 42.68 $\pm$ 0.27 & 30.46 $\pm$ 0.26 & & 5.00 $\pm$ 0.64 \\
PS-PLT$_{h^2 + beam}$ & 33.11 $\pm$ 0.12 & 36.23 $\pm$ 0.28 & 37.92 $\pm$ 0.39 & 65.82 $\pm$ 0.23 & 42.58 $\pm$ 0.27 & 30.41 $\pm$ 0.25 & & \textbf{1.13 $\pm$ 0.06} \\

\specialrule{0.94pt}{0.4ex}{0.65ex}
& \multicolumn{8}{c}{\wikipedia} \\
\midrule
PLT$_{log}$ & 26.28 $\pm$ 0.09 & 30.93 $\pm$ 0.12 & 34.15 $\pm$ 0.14 & 67.50 $\pm$ 0.27 & 48.26 $\pm$ 0.20 & 37.74 $\pm$ 0.15 & \multirow{3}{*}{$\begin{rcases}\\ \\ \\ \end{rcases}$ 46.17 $\pm$ 0.32} & 26.40 $\pm$ 0.64 \\
PS-PLT$_{log + A^*}$ & \textbf{34.12 $\pm$ 0.10} & \textbf{35.70 $\pm$ 0.12} & \textbf{38.14 $\pm$ 0.13} & 67.53 $\pm$ 0.21 & 48.68 $\pm$ 0.15 & 38.23 $\pm$ 0.12 & & 60.01 $\pm$ 5.58 \\
PS-PLT$_{log + beam}$ & 34.11 $\pm$ 0.11 & 35.65 $\pm$ 0.13 & 38.06 $\pm$ 0.15 & 67.54 $\pm$ 0.23 & 48.63 $\pm$ 0.17 & 38.17 $\pm$ 0.14 & & \textbf{4.81 $\pm$ 0.13} \\
PLT$_{h^2}$ & 26.71 $\pm$ 0.08 & 31.62 $\pm$ 0.14 & 34.91 $\pm$ 0.19 & 68.28 $\pm$ 0.23 & \textbf{49.13 $\pm$ 0.20} & \textbf{38.33 $\pm$ 0.18} & \multirow{3}{*}{$\begin{rcases}\\ \\ \\ \end{rcases}$ \textbf{27.05 $\pm$ 0.19}} & 10.21 $\pm$ 0.07 \\
PS-PLT$_{h^2 + A^*}$ & 33.77 $\pm$ 0.11 & 35.64 $\pm$ 0.21 & 37.37 $\pm$ 0.31 & 68.54 $\pm$ 0.23 & 48.45 $\pm$ 0.26 & 37.03 $\pm$ 0.27 & & 21.20 $\pm$ 4.31 \\
PS-PLT$_{h^2 + beam}$ & 33.75 $\pm$ 0.15 & 35.60 $\pm$ 0.29 & 37.33 $\pm$ 0.41 & \textbf{68.57 $\pm$ 0.32} & 48.44 $\pm$ 0.36 & 37.07 $\pm$ 0.36 & & 4.91 $\pm$ 0.37 \\

\specialrule{0.94pt}{0.4ex}{0.65ex}
& \multicolumn{8}{c}{\amazon} \\
\midrule
PLT$_{log}$ & 26.31 $\pm$ 0.06 & 30.22 $\pm$ 0.08 & 33.83 $\pm$ 0.10 & \textbf{45.01 $\pm$ 0.12} & \textbf{40.21 $\pm$ 0.11} & \textbf{36.72 $\pm$ 0.10} & \multirow{3}{*}{$\begin{rcases}\\ \\ \\ \end{rcases}$ 1.92 $\pm$ 0.01} & 12.06 $\pm$ 0.05 \\
PS-PLT$_{log + A^*}$ & \textbf{31.14 $\pm$ 0.07} & \textbf{33.45 $\pm$ 0.09} & \textbf{35.60 $\pm$ 0.11} & 43.71 $\pm$ 0.10 & 39.72 $\pm$ 0.09 & 36.60 $\pm$ 0.10 & & 20.40 $\pm$ 0.45 \\
PS-PLT$_{log + beam}$ & 30.95 $\pm$ 0.07 & 33.13 $\pm$ 0.11 & 35.14 $\pm$ 0.14 & 43.48 $\pm$ 0.11 & 39.40 $\pm$ 0.11 & 36.21 $\pm$ 0.13 & & 1.57 $\pm$ 0.15 \\
PLT$_{h^2}$ & 26.22 $\pm$ 0.08 & 29.89 $\pm$ 0.12 & 33.12 $\pm$ 0.16 & 44.78 $\pm$ 0.17 & 39.75 $\pm$ 0.16 & 35.97 $\pm$ 0.16 & \multirow{3}{*}{$\begin{rcases}\\ \\ \\ \end{rcases}$ \textbf{1.44 $\pm$ 0.01}} & 4.56 $\pm$ 0.14 \\
PS-PLT$_{h^2 + A^*}$ & 29.92 $\pm$ 0.09 & 32.23 $\pm$ 0.12 & 34.21 $\pm$ 0.17 & 43.57 $\pm$ 0.15 & 38.95 $\pm$ 0.13 & 35.33 $\pm$ 0.16 & & 6.59 $\pm$ 0.04 \\
PS-PLT$_{h^2 + beam}$ & 29.82 $\pm$ 0.09 & 32.02 $\pm$ 0.12 & 33.91 $\pm$ 0.17 & 43.45 $\pm$ 0.15 & 38.75 $\pm$ 0.14 & 35.08 $\pm$ 0.16 & & \textbf{1.17 $\pm$ 0.04} \\

\bottomrule
\end{tabular}

\label{tab:pw-ps-plt-vs-sota}
\end{center}
\end{table*}